\documentclass{article}

\usepackage[preprint]{neurips_2026}
\workshoptitle{E-Values: From Statistics to ML}

\usepackage[utf8]{inputenc}
\usepackage[T1]{fontenc}
\usepackage{hyperref}
\usepackage{url}
\usepackage{booktabs}
\usepackage{amsfonts}
\usepackage{amssymb}
\usepackage{amsmath}
\usepackage{mathtools}
\usepackage{nicefrac}
\usepackage{microtype}
\usepackage{xcolor}
\usepackage{graphicx}
\usepackage{tikz}
\usepackage{dsfont}
\usepackage{amsthm}
\mathtoolsset{showonlyrefs}

\newtheorem{theorem}{Theorem}
\newtheorem{proposition}[theorem]{Proposition}
\newtheorem{definition}[theorem]{Definition}
\newtheorem{remark}[theorem]{Remark}
\newtheorem{lemma}[theorem]{Lemma}

\newcommand{\ZZ}{\mathbf{Z}}
\newcommand{\XX}{\mathbf{X}}
\newcommand{\YY}{\mathbf{Y}}
\newcommand{\UU}{\mathbf{U}}

\usepackage{todonotes}

\title{Information-Geometric Blind Spots of Conformal Test Martingales}

\author{%
  Johan Hallberg Szabadv\'ary \\
  Department of Mathematics, Stockholm University\\
  Department of Computing, J\"onk\"oping University \\
  \texttt{johan.hallberg.szabadvary@math.su.se}
}

\begin{document}

\maketitle

\begin{abstract}
Conformal test martingales (CTMs) are e-processes for sequentially testing exchangeability via p-values produced by a conformity function $A$. However, any distribution shift that preserves the push-forward $P\circ A^{-1}$ leaves the p-value distribution invariant, rendering it invisible to every CTM based on $A$. We formalise this blind spot through the \emph{$A$-equivalence class} $[P]_A$ and show that, for any target distribution $T$, there is a unique closest element $Q^\ast\in[P]_A$ in KL divergence---the \emph{$A$-cryptic I-projection}---with explicit density $q^\ast(z) = t(z)\,p_S(A(z))/t_S(A(z))$ and a Pythagorean decomposition into detectable and cryptic components. When $A$ is the Neyman-Pearson likelihood ratio, the cryptic component vanishes, connecting our result to optimal e-value design. We demonstrate the theory on high-dimensional image data by empirically constructing this I-projection. By conditionally resampling EMNIST letters to identically match the conformity scores of MNIST digits, we demonstrate that a distribution can be shifted to a disjoint support while leaving the test martingale mathematically blinded.
\end{abstract}

\section{Introduction}\label{sec:intro}

Conformal prediction \citep{alrw2} constructs prediction sets with finite-sample validity under the sole assumption of exchangeability. When a fixed \emph{conformity function} $A\colon\ZZ\to\mathbb{R}$ scores examples, the \emph{conformal transducer} maps the data sequence to p-values that are independent and uniform on $[0,1]$ under exchangeability \citep[Thm. 12.1]{alrw2}. A \emph{conformal test martingale} (CTM) is a non-negative martingale built from these p-values; by Ville's inequality \citep{ville1939etude}, martingale growth constitutes evidence against exchangeability \citep{vovk2003testing}. CTMs are thus e-processes \citep{ramdas2022testing} for sequential testing of exchangeability. 

This construction bypasses the impossibility result of \citet{ramdas2022testing}: any non-negative supermartingale adapted to the canonical data filtration is necessarily non-increasing under exchangeability, and hence powerless. CTMs circumvent this by adapting to the strictly coarser filtration $\mathcal{G}_n = \sigma(p_1,\dots,p_n)$. However, this filtration reduction comes at a cost: a shift from $P$ to $Q$ satisfying $P\circ A^{-1} = Q\circ A^{-1}$ leaves the p-value distribution invariant (Proposition \ref{prop:collapse}), rendering it invisible to \emph{every} CTM based on $A$. 

\textbf{Contributions.} We formalise this blind spot and provide an information-geometric characterisation:
\begin{enumerate}
    \item We define the \emph{$A$-equivalence class} $[P]_A$ and show that for independent data streams, a shift is undetectable if the score push-forwards coincide (Section~\ref{sec:background}).
    \item We prove that for any target $T$, there is a unique closest element of $[P]_A$ in KL divergence, the \emph{$A$-cryptic I-projection} $Q^*$. (Section~\ref{sec:iprojection}).
    \item We demonstrate the $A$-cryptic I-projection on MNIST by projecting EMNIST letters onto the digits and verify empirically that the resulting letters are not detected (Section~\ref{sec:experiments}).
\end{enumerate}

\section{Background}
\label{sec:background}

We assume throughout the paper that $\ZZ = \XX \times \YY$ is a standard Borel space. A \emph{conformity function} is a measurable map $A\colon\ZZ\to\mathbb{R}$, and the \emph{conformal transducer} (CT) determined by $A$ outputs p-values
\begin{equation}\label{eq:transducer}
    p_k = \frac{|\{i \le k : \alpha_i < \alpha_k\}| + \tau_k\,|\{i \le k : \alpha_i = \alpha_k\}|}{k},
\end{equation}
where $\alpha_i = A(z_i)$ and $\tau_i \sim \UU[0,1]$ are independent tie-breakers. Under exchangeability, the $p_k$ are IID\ $\UU[0,1]$ \citep[Thm. 12.1]{alrw2}. A \emph{conformal test martingale} is $S_n = F(p_1,\dots,p_n)$, where $F$ is a betting martingale satisfying $\int_0^1 F(u_1,\dots,u_{n-1},u)\,du = F(u_1,\dots,u_{n-1})$ and $F(\square)=1$. We say that the CTM $S$ is determined by the betting martingale $F$ and the conformity function $A$. The intuitive picture is that the betting martingale gambles against the hypothesis that the conformal p-values are IID uniform.

\begin{proposition}\label{prop:collapse}
    Let $z_1,z_2,\dots$ be independent with $z_t \sim P_t$. If $P_t\circ A^{-1} = P_1\circ A^{-1}$ for all $t$, then the CT determined by $A$ outputs IID $\UU[0,1]$ p-values.
\end{proposition}
\begin{proof}
    Let $\alpha_t = A(z_t)$. Under independence, the $\alpha_t$ are mutually independent with marginals $\mu_t = P_t\circ A^{-1}$. If $\mu_t = \mu_1$ for all $t$, the scores are IID and hence exchangeable. The CT produces IID\ $\UU[0,1]$ p-values by \citet[Thm. 12.1]{alrw2} (apply the identity function as conformity measure to $\alpha_1,\alpha_2,\dots$).
\end{proof}
Proposition \ref{prop:collapse} can be strengthened to include the reverse implication that, assuming independence, IID uniform p-values imply equal push-forwards, but this is beyond the scope of this short paper.
\begin{definition}[$A$-equivalence class]\label{def:A-equiv}
The \emph{$A$-equivalence class of $P$} is $[P]_A = \{Q : Q\circ A^{-1} = P\circ A^{-1}\}$.
\end{definition}

Proposition~\ref{prop:collapse} implies that for independent data streams (including change-point scenarios), a shift from $P$ to $Q$ is invisible to every CTM based on $A$ if and only if $Q\in[P]_A$. Such pairs $(P,Q)$ are called \emph{$A$-cryptic pairs}.

\section{The $A$-Cryptic I-Projection}\label{sec:iprojection}

Given a target distribution $T$ (e.g., a post-shift distribution), we seek the element of $[P]_A$ closest to $T$ in KL divergence. This characterizes the ``worst-case cryptic adversary'' in the direction of $T$.

\begin{theorem}[Existence and uniqueness of the $A$-cryptic I-projection]\label{thm:iprojection}
Let $P_S = P\circ A^{-1}$, $T_S = T\circ A^{-1}$, and assume $D_{\mathrm{KL}}(P_S\|T_S)<\infty$. There exists a unique $Q^*\in[P]_A$ minimising $D_{\mathrm{KL}}(Q\|T)$ subject to $Q\in[P]_A$, given by
\begin{equation}\label{eq:Qstar}
Q^*(E) = \int_{\mathbb{R}} T_\alpha(E)\,dP_S(\alpha),
\end{equation}
where $\{T_\alpha\}$ is the disintegration of $T$ along $A$ (see Lemma \ref{lemma:disintegration} in Appendix \ref{app:proofs}).
\end{theorem}
\begin{proof}[Proof sketch]
By the chain rule for KL divergence, $D_{\mathrm{KL}}(Q\|T) = D_{\mathrm{KL}}(Q_S\|T_S) + \mathbb{E}_{Q_S}[D_{\mathrm{KL}}(Q_\alpha\|T_\alpha)]$. The constraint $Q\in[P]_A$ forces $Q_S = P_S$, making the first term constant. The second term is minimised pointwise by setting $Q_\alpha = T_\alpha$, which yields $Q^*$.
\end{proof}

When $P$ and $T$ admit densities $p$ and $t$ with respect to a reference measure, $Q^*$ has an explicit density as follows:

\begin{theorem}[Density of the $A$-cryptic I-projection]\label{thm:density}
Under the assumptions of Theorem~\ref{thm:iprojection}, if $\ZZ$ is equipped with a $\sigma$-finite reference measure $\nu$, and $P$ and $T$ have densities $p$ and $t$ w.r.t. $\nu$, then $Q^*$ has density
\begin{equation}\label{eq:density}
q^*(z) = t(z)\,\frac{p_S(A(z))}{t_S(A(z))},
\end{equation}
w.r.t. $\nu$, where $p_S$ and $t_S$ are the score-marginal densities of $P$ and $T$ (w.r.t. some common $\sigma$-finite dominating measure on $\mathbb{R}$).
\end{theorem}

\textbf{Pythagorean identity.} The $A$-cryptic I-projection satisfies the decomposition
\begin{equation}\label{eq:pythagoras}
D_{\mathrm{KL}}(P\|T) = D_{\mathrm{KL}}(P\|Q^*) + D_{\mathrm{KL}}(Q^*\|T),
\end{equation}
illustrated in Figure~\ref{fig:pythagorean}. The \emph{detectable} component $D_{\mathrm{KL}}(Q^*\|T) = D_{\mathrm{KL}}(P_S\|T_S)$ captures the score-marginal mismatch visible to a CTM, while the \emph{cryptic} component $D_{\mathrm{KL}}(P\|Q^*) = \mathbb{E}_{P_S}[D_{\mathrm{KL}}(P_\alpha\|T_\alpha)]$ quantifies the fibre-conditional divergence hidden from it. The $A$-equivalence class $[P]_A$, defined by fixing the marginal distribution of $A(Z)$, forms a variation-closed convex set in the sense of \citet{csiszar1975idivergence}; the decomposition \eqref{eq:pythagoras} is a special case of Csisz\'ar's I-projection theorem onto linear families.

\begin{figure}[ht]
    \centering
    \begin{tikzpicture}[scale=1.3, >=latex]
        \draw[thick, fill=blue!5, draw=blue!40] (-1.8, 0) -- (2.8, 0) -- (3.6, 1.3) -- (-1.0, 1.3) -- cycle;
        \node[blue!70!black] at (3.1, 0.25) {$[P]_A$};
        \coordinate (P) at (0, 0.65);
        \coordinate (Qstar) at (1.8, 0.65);
        \coordinate (T) at (1.8, 2.5);
        \draw[thick, dashed, ->] (P) -- (Qstar) node[midway, below, font=\small] {$D(P\|Q^*)$ \scriptsize(Cryptic)};
        \draw[thick, ->] (Qstar) -- (T) node[midway, right, font=\small] {$D(Q^*\|T)$ \scriptsize(Visible)};
        \draw[thick, gray, ->] (P) -- (T) node[midway, above left, font=\small] {$D(P\|T)$};
        \filldraw[black] (P) circle (1.5pt) node[below left] {$P$};
        \filldraw[red] (T) circle (1.5pt) node[above] {$T$};
        \filldraw[violet] (Qstar) circle (1.5pt) node[below right] {$Q^*$};
        \draw[thin] (1.6, 0.65) -- (1.6, 0.85) -- (1.8, 0.85);
    \end{tikzpicture}
    \caption{The $A$-cryptic I-projection $Q^*$ satisfies a Pythagorean identity.}
    \label{fig:pythagorean}
\end{figure}
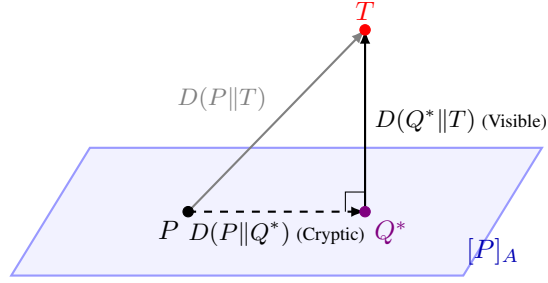

\begin{remark}[Connection to optimal e-value design]\label{rem:NP}
If the conformity function is the Neyman--Pearson likelihood ratio $A(z) = t(z)/p(z)$, then $p_S(\alpha)/t_S(\alpha) = 1/\alpha$, and the $A$-cryptic I-projection density \eqref{eq:density} reduces to $q^*(z) = t(z)\cdot p(z)/t(z) = p(z)$. That is, $Q^* = P$: the cryptic component vanishes entirely. This aligns with the optimality result of \citet{hore2026conformal}, who prove that the likelihood ratio conformity function maximises detection power.
\end{remark}
An exact closed-form realisation of the $A$-cryptic I-projection in a bivariate Gaussian setting is provided in Appendix~\ref{app:gaussian}.

\section{Example: $A$-cryptic I-projection of EMNIST on MNIST}\label{sec:experiments}

We demonstrate that the blind spot $[P]_A$ can be exploited to construct the $A$-cryptic I-projection $Q^*$ on high-dimensional image data. We trained a multiclass logistic regression on MNIST digits \citep{lecun1988mnist} (92.6\% accuracy), yielding the conformity function $A(x) = \max_y\,\mathrm{softmax}(Wx+b)_y$. The true $A$-cryptic I-projection requires minimising the Kullback--Leibler divergence to a target distribution $T$ subject to the constraint $Q^*_A = P_A$. 

Let $T$ be the distribution of EMNIST letters \citep{cohen2017emnist}. Following the disintegration form in Theorem~\ref{thm:density}, we construct $Q^*$ not by perturbing pixels but by empirically reweighting $T$. For each digit $x \sim P$ in the test stream, we compute its conformity score $\alpha = A(x)$. We then search the EMNIST dataset for the $k=5$ letters whose scores minimise the absolute difference $|A(x') - \alpha|$, and sample one uniformly at random. This nearest-neighbor approach empirically samples from the conditional target distribution $T(\cdot \mid A(x)=\alpha)$. By directly resampling existing letters, $Q^*$ is strictly confined to the natural target manifold, completely eliminating numerical artefacts.

Figure~\ref{fig:mnist} illustrates the results. Because the matched sequence of scores $A(Q^*)$ tightly approximates $A(P)$, the CTM---a Sleeper/Stayer martingale \citep[Section 9.2.4]{alrw2}---for the cryptic stream $Q^*$ is entirely indistinguishable from the digit baseline. In contrast, the naive EMNIST stream $T$ is trivially detected, and the CTM grows exponentially.

This illustrates a profound diagnostic vulnerability: any conformity function aggressively compresses high-dimensional data (here, mapping $\mathbb{R}^{784} \to \mathbb{R}$), discarding large amounts of orthogonal information. This massive dimensionality reduction creates a dense equivalence class $[P]_A$ where the underlying data distribution can be shifted to an essentially disjoint support (from digits to letters) while leaving the test martingale mathematically blinded.

\begin{figure}[ht]
    \centering
    \includegraphics[width=\linewidth]{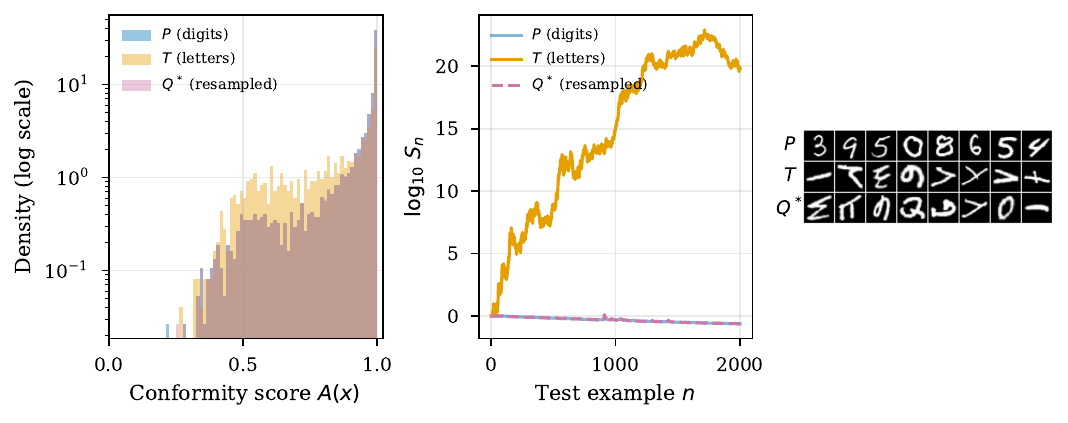}
    \caption{Empirical $A$-cryptic I-projection. \textbf{Left:} Conformity score densities. By construction, the cryptic projection $Q^*$ empirically matches the null $P$. \textbf{Middle:} Sleeper/Stayer martingales. The naive shift to $T$ is rapidly detected, whereas $Q^*$ perfectly tracks the null hypothesis. \textbf{Right:} Image triplets. The $Q^*$ samples are pristine letters, demonstrating that a distribution can be shifted to a disjoint support within the equivalence class $[P]_A$ without altering the test statistic.}
    \label{fig:mnist}
\end{figure}

\section{Discussion}\label{sec:discussion}

Every conformity function $A$ induces a structural blind spot---the equivalence class $[P]_A$---within which distribution shifts are invisible to all CTMs based on $A$. The $A$-cryptic I-projection decomposes any shift into a detectable score-marginal component and a cryptic fibre-conditional component, with the Pythagorean identity \eqref{eq:pythagoras} quantifying each.

\textbf{Adversarial implications.} An adversary who knows $A$ can craft distribution shifts that evade CTM-based monitoring entirely. The MNIST null-space injection demonstrates this concretely: visually altered images are invisible to the CTM. Unlike classical adversarial examples that target the classifier's predictions, these attacks target the \emph{monitoring system} itself---the conformal machinery that certifies that the data distribution has not changed. This represents a distinct threat model for deployed conformal prediction systems.

\textbf{Optimal conformity function design.} Remark~\ref{rem:NP} shows that the likelihood ratio $A(z) = t(z)/p(z)$ eliminates the cryptic component entirely ($Q^*=P$), connecting our blind-spot characterisation to the optimal e-value construction of \citet{hore2026conformal}. In practice, $t$ is unknown, but this result provides an information-geometric justification for designing conformity functions that approximate the likelihood ratio. However, the likelihood ratio score only ``closes down'' one specific direction where $A$-cryptic I-projection collapses. We speculate that ensembles of targeted likelihood ratio conformity functions may provide robustness in critical information-geometric directions.

\textbf{Inductive conformal prediction.} Our theory deals with \emph{fixed} conformity functions, which theoretically align with the inductive CP framework applied in the semi-online mode \citep[Section 4.2.2]{alrw2}. Full (or \emph{transductive}) CP involves updating the conformity function as new examples arrive, thereby transforming the fixed blind spot $[P]_A$ into a moving target. 

\textbf{Limitations.} The theory presented here assumes independent observations; the full characterisation of general (possibly dependent) processes is developed in an extended version.\footnote{An extended version develops the full measure-theoretic characterisation of $A$-cryptic pairs via disintegration of measures, stabiliser subgroups of the automorphism group, and efficiency analysis of conformal predictors under cryptic shifts.} The MNIST projection example is exact for linear conformity functions. For nonlinear models, the level sets of $A$ are curved manifolds, and exploiting them requires numerical methods such as constrained Langevin dynamics.

\textbf{Ethics statement.} While we demonstrate the utility of $A$-cryptic projections for theoretical analysis, we acknowledge that these results reveal a potential vulnerability in CTM-based monitoring. An adversary with knowledge of $A$ could potentially evade detection systems. We believe publishing this blind spot is essential for the development of robust, ``adversary-aware'' conformity functions, possibly based on ensembling, as discussed above.

\bibliography{references}
\bibliographystyle{plainnat}

\newpage
\appendix
\section{Proofs}\label{app:proofs}

The following lemma provides the ``coordinate system'' for our analysis. It allows us to view any probability measure $P$ on $\ZZ$ as a mixture of measures supported on the fibres (level sets) of the conformity function.
\begin{lemma}[Disintegration of Measures]\label{lemma:disintegration}
    Let $(\ZZ,\mathcal{F})$ be a standard Borel space and $A:\ZZ\to\mathbb{R}$ be a measurable map. For a probability measure $P$ on $\ZZ$, let the push-forward measure $P_S := P \circ A^{-1}$ be the marginal distribution of the conformity scores. Then there exists a $P_S$-almost everywhere unique family of probability measures $\{P_{\alpha}\}_{\alpha\in\mathbb{R}}$, called the disintegration kernel, satisfying
    \begin{itemize}%[label=(\roman*)]
        \item 
        For $P_S$-a.e $\alpha$, the measure $P_{\alpha}$ assigns full probability mass to the the fibre 
        $$
            A^{-1}(\alpha) = \{z\in\ZZ: A(z) = \alpha\}.
        $$
        That is, $P_{\alpha}(A^{-1}(\alpha)) = 1.$
        \item 
        For any event $E\in\mathcal{F}$, the probability $P(E)$ decomposes as
        $$
            P(E) = \int_{\mathbb{R}}P_{\alpha}(E)dP_S(\alpha).
        $$
    \end{itemize}
\end{lemma}

\begin{proof}
    By definition, the standard Borel space $(\ZZ, \mathcal{F})$ can be equipped with a complete separable metric whose induced topology generates $\mathcal{F}$. Because $P$ is a probability measure on this generated $\sigma$-algebra, it is a finite Borel measure. By Ulam's theorem, every finite Borel measure on a complete separable metric space is tight and thus qualifies as a Radon measure \citep[Theorem 1.3]{billingsley1999convergence}. 
    Furthermore, the target space $\mathbb{R}$ is equipped with the Borel $\sigma$-algebra $\mathcal{B}(\mathbb{R})$, which is countably generated and contains all singleton sets.

    As $P$ is a finite Radon measure, we can apply the disintegration theorem of \citet{chang1997conditioning}. According to \citet[Theorem 1]{chang1997conditioning}, there exists a disintegration of $P$ with respect to the measurable map $A$ and the image measure $P_S$, and this family of disintegrating measures is uniquely determined $P_S$-almost everywhere. Furthermore, by \citet[Theorem 2 (iii)]{chang1997conditioning}, because $P_S$ is exactly the image measure of $P$, the disintegrating measures $P_\alpha$ are probability measures for $P_S$-almost all $\alpha$. By the definition of disintegration, these measures are concentrated on the fibres $A^{-1}(\alpha)$. The integral decomposition follows directly from this definition. 
\end{proof}

\begin{proof}[Proof of Theorem~\ref{thm:iprojection}]
    We seek to minimise
    $$
        D_{KL}(Q\parallel T) = \int_{\ZZ}\log\bigg(\frac{dQ}{dT}(z)\bigg)dQ(z)
    $$
    subject to the constraint $Q\in[P]_A$, which is equivalent to the marginal constraint $Q\circ A^{-1} = P_S$.
    
    By the chain rule for KL divergence (see e.g. \citet[Theorem 2.15]{Polyanskiy_Wu_2025}),
    $$
        D_{KL}(Q \parallel T) = D_{KL}(Q_S \parallel T_S) + D(Q_{Z\mid S} \parallel T_{Z\mid S} \mid Q_S),
    $$
    where the conditional term is defined as the expected divergence over the fibres:
    $$
    D(Q_{Z\mid S} \parallel T_{Z\mid S} \mid Q_S) = \mathbb{E}_{s \sim Q_S} \left[ D_{KL}(Q_{\alpha} \parallel T_{\alpha}) \right] = \int_{\mathbb{R}} D_{KL}(Q_{\alpha} \parallel T_{\alpha})  dQ_S(\alpha).
    $$
    The constraint implies that for any feasible $Q$, we must have $Q_S = P_S$ almost everywhere. Furthermore, the assumption $D_{KL}(P_S \parallel T_S) < \infty$ strictly implies that $P_S \ll T_S$. Consequently, any $T_S$-null set is also a $P_S$-null set. This guarantees that the target disintegration kernel $T_{\alpha}$, which is uniquely defined $T_S$-almost everywhere by Lemma \ref{lemma:disintegration}, is well-defined $P_S$-almost everywhere. This allows us to validly substitute $P_S$ into the decomposition
    $$
        D_{KL}(Q \parallel T) = D_{KL}(P_S \parallel T_S) + \int_{\mathbb{R}} D_{KL}(Q_{\alpha} \parallel T_{\alpha})  dP_S(\alpha).
    $$
    The term $D_{KL}(P_S \parallel T_S)$ is determined solely by the fixed distributions $P$ and $T$, so it is constant with respect to the optimisation variable $Q$. By assumption, it is finite. The term $\int_{\mathbb{R}} D_{KL}(Q_{\alpha} \parallel T_{\alpha})  dP_S(\alpha)$ involves the disintegration kernels $Q_{\alpha}$ and $T_{\alpha}$. From the information inequality \citep[Theorem 2.3]{Polyanskiy_Wu_2025}, $D_{KL}(Q_{\alpha} \parallel T_{\alpha})\geq0$ with equality if and only if $Q_{\alpha}=T_{\alpha}$ for $P_S$-almost all $\alpha$.

    This minimisation logic uniquely identifies the components of the candidate measure $Q^*$. The constraint enforces $Q^*_S = P_S$, and setting the disintegration kernels $Q_{\alpha}=T_{\alpha}$ minimises the integral term.
    
    By construction, $Q^*\circ A^{-1} = P_S = P\circ A^{-1}$, which implies $Q^*\in[P]_A$. Thus, $Q^*$ is a feasible solution to the minimisation problem. Moreover, by Lemma \ref{lemma:disintegration}, $Q^*$ is uniquely determined by its marginal and its disintegration kernel, which means that the measure $Q^*$ defined by
    \begin{equation}\label{eq:I-projection-disintegration}
        Q^*(E) = \int_{\mathbb{R}} T_{\alpha}(E)  dP_S(\alpha)
    \end{equation}
    is the unique minimiser.
\end{proof}

\begin{proof}[Proof of Theorem~\ref{thm:density}]
    We begin by finding the density of the measure $Q^*$ relative to $T$. Let $w(\alpha) = \frac{dP_S}{dT_S}(\alpha)$ be the Radon-Nikodym derivative of the marginals. Because $P_S$ and $T_S$ possess densities $p_S$ and $t_S$ with respect to a common dominating measure, we have $w(\alpha) = \frac{p_S(\alpha)}{t_S(\alpha)}$ almost everywhere.
    Substituting $dP_S(\alpha) = w(\alpha) dT_S(\alpha)$ into \eqref{eq:I-projection-disintegration}, we get
    $$
        Q^*(E) = \int_{\mathbb{R}} T_{\alpha}(E)  w(\alpha)  dT_S(\alpha).
    $$
    This integral can be rewritten as an expectation over the target distribution $T$. Recall that $T_{\alpha}(E) = P_T(Z \in E \mid S=\alpha)$. Thus,
    $$
    \begin{aligned}
            Q^*(E) &= \int_{\mathbb{R}} P_T(Z \in E \mid S=\alpha)  w(\alpha)  dT_S(\alpha) \\
            &= \mathbb{E}_{T} \left[ \mathds{1}_E(Z)  w(S) \right] \quad \text{(Law of Total Expectation)} \\
            &= \int_Z \mathds{1}_E(z)  w(A(z))  dT(z).
            \end{aligned}
    $$
    This shows that $Q^*$ is absolutely continuous with respect to $T$ with Radon-Nikodym derivative
    $$
        \frac{dQ^*}{dT}(z) = w(A(z)) = \frac{p_S(A(z))}{t_S(A(z))}.
    $$
    Finally, the density with respect to the reference measure $\nu$ is
    $$
        q^*(z) = \frac{dQ^*}{dT}(z) \frac{dT}{d\nu}(z) = t(z)\frac{p_S(A(z))}{t_S(A(z))},
    $$
    which is the stated form.
\end{proof}

\section{Analytic Gaussian $A$-cryptic I-Projection}\label{app:gaussian}

    To confirm the exact mathematical mechanics of this projection and eliminate estimation artefacts, we derive a closed-form analytic projection in an idealised parametric setting. We take inspiration from \citet[Chapter 3]{alrw2} and use the idealised \emph{conditional density conformity measure} (CDCM), the true conditional density function under $P$. As this is a fixed function $\ZZ\to\mathbb{R}$, it aligns naturally with our theoretical framework using conformity functions. 

    We define the example space $\ZZ = \mathbb{R}^2$. The null distribution $P$ is specified as a zero-mean bivariate Gaussian distribution $P = \mathcal{N}(0, \Sigma)$ with covariance matrix 
    $$
        \Sigma = \begin{pmatrix} 1 & \rho \\ \rho & 1 \end{pmatrix}
    $$
    where $\rho=0.8$. We introduce a non-cryptic target distribution $T = \mathcal{N}(\mu_{T}, \Sigma)$, where $\mu_T\neq0$, indicating a mean shift. Note that while the covariance structure is preserved, the shift in the mean introduces a significant bias in the conditional expectation $\mathbb{E}[Y|X]$.

    We use the idealised CDCM as conformity function, which is given by
    \begin{equation}
        A(x,y) = f_{P}(y|x) = \frac{1}{\sqrt{2\pi(1-\rho^{2})}} \exp\left( -\frac{(y - \rho x)^{2}}{2(1-\rho^{2})} \right).
    \end{equation}
    This choice is theoretically instructive because the $\sigma$-algebra generated by a continuous conformity function $\sigma(A)$ is generally a proper sub-$\sigma$-algebra of $\mathcal{F}$. Here, the score $S = A(z)$ is a strictly monotonic function of the squared standardised residual $R^{2}(z) = \frac{(y - \rho x)^{2}}{1-\rho^{2}}$. Consequently, $A$ generates the same $\sigma$-algebra as $R^{2}$. The information lost by this conformity function is the sign of the residual and marginal distribution of $X$; the conformity function is blind to any shift that preserves the distribution of the error magnitude.

    According to Theorem \ref{thm:density}, the density of the I-projection is given by \eqref{eq:density}, and 
    because the mapping between the score $S$ and squared residual $R^{2}$ is bijective, the likelihood ratio of the scores is identical to the likelihood ratio of the residuals. Thus, we can compute the weights analytically by characterising the distributions of $R^{2}$ under $P$ and $T$. Under $P$, the standardised residual follows a standard normal distribution, implying that $R^{2} \sim \chi^{2}_{1}$ (central chi-squared distribution with 1 degree of freedom). Under $T$, owing to the mean shift, the residual $Y - \rho X$ has a non-zero mean. Consequently, $R^{2}$ follows a non-central chi-squared distribution $R^{2} \sim \chi'^{2}_{1}(\lambda)$, where the non-centrality parameter $\lambda$ is determined by the conditional bias at the target mean.
    
    The I-projection is therefore explicitly constructed by reweighting $T$ as 
    \begin{equation}\label{eq:gauss-q-density}
        q^{*}(z) = t(z) \frac{f_{\chi^{2}_{1}}(R^{2}(z))}{f_{\chi'^{2}_{1}(\lambda)}(R^{2}(z))}.
    \end{equation}
    While \eqref{eq:gauss-q-density} defines the exact density, evaluating it via rejection sampling is unnecessary. We can construct an exact generative sampler by exploiting the independence structure of the $I$-projection. Let the standardized residual be $R(z) = \frac{y - \rho x}{\sqrt{1-\rho^2}}$, which can be decomposed into a magnitude $\vert{}R\vert{}$ and a sign $S \in \{-1, 1\}$.Because the conformity function $A$ restricts only the distribution of $\vert{}R\vert{}$, the $I$-projection $Q^*$ must adopt the distribution of $\vert{}R\vert{}$ under $P$, which is a standard half-normal distribution. To minimize the Kullback-Leibler divergence to $T$, $Q^*$ retains the target's marginal distribution for $X$ and the target's conditional distribution for the sign given the magnitude, $T(S \mid \vert{}R\vert{})$.We can therefore sample exactly from $Q^*$ using the following procedure:
    \begin{enumerate}
        \item 
        Draw the feature $X \sim \mathcal{N}(\mu_{T_X}, 1)$.
        \item 
        Draw the residual magnitude $\vert{}R\vert{}$ from the null distribution $P$, equivalent to $\vert{}R\vert{} = \vert{}Z\vert{}$ where $Z \sim \mathcal{N}(0, 1)$.
        \item 
        Draw the sign $S \in \{-1, 1\}$ from the conditional target distribution. Under $T$, $R \sim \mathcal{N}(\mu_R, 1)$ where $\mu_R = \frac{\mu_{T_Y} - \rho \mu_{T_X}}{\sqrt{1-\rho^2}}$. The conditional probability of a positive sign is exactly $P_T(S=1 \mid \vert{}R\vert{}) = (1 + \exp(-2 \mu_R \vert{}R\vert{}))^{-1}$.
        \item 
        Reconstruct the response $Y = \rho X + S \vert{}R\vert{} \sqrt{1-\rho^2}$.
    \end{enumerate}
    \begin{figure}[ht]
        \centering
        \includegraphics[width=\linewidth]{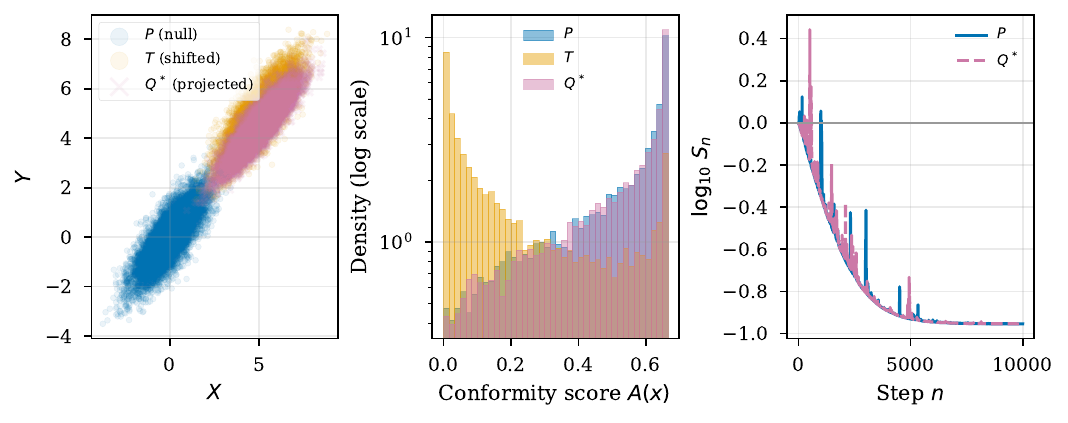}
        \caption{Exact $A$-cryptic I-projection in the Gaussian setting. Left: $T$ with $\mu_T = [5,5]$ (orange) is distant from $P$ (blue) in feature space, but $Q^*$ (purple) reweights $T$ to recover $P$'s score distribution (middle). The CTM for $Q^*$ remains non-increasing (right).}
        \label{fig:gaus-proj-5.0,5.0}
    \end{figure}
    \begin{figure}[ht]
        \centering
        \includegraphics[width=\linewidth]{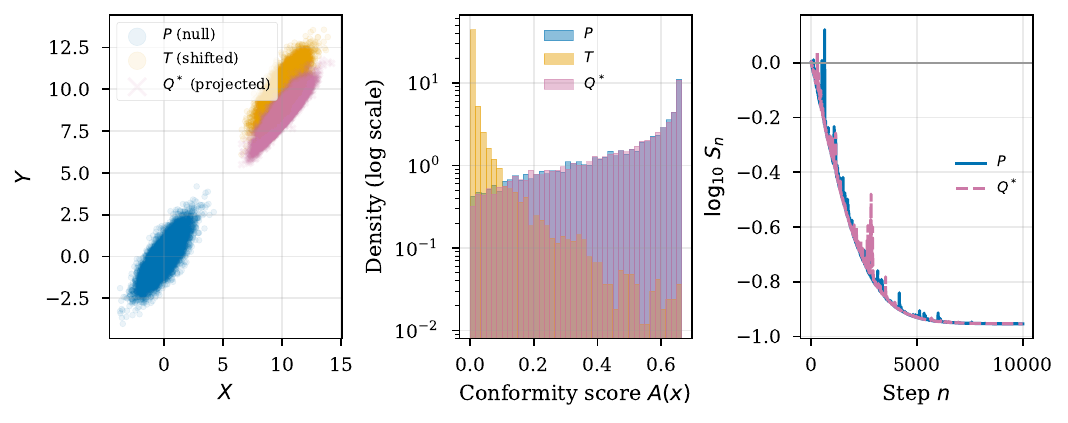}
        \caption{Exact $A$-cryptic I-projection in the Gaussian setting. Left: $T$ with $\mu_T = [10,10]$ (orange) is distant from $P$ (blue) in feature space, but $Q^*$ (purple) reweights $T$ to recover $P$'s score distribution (middle). The CTM for $Q^*$ remains non-increasing (right).}
        \label{fig:gaus-proj-10.0,10.0}
    \end{figure}
    \begin{figure}[ht]
        \centering
        \includegraphics[width=\linewidth]{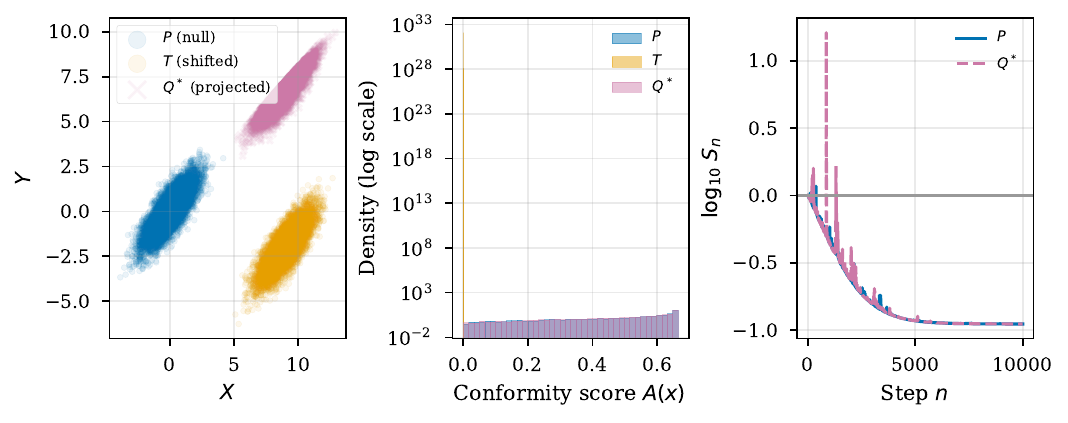}
        \caption{Exact $A$-cryptic I-projection in the Gaussian setting. Left: $T$ with $\mu_T = [9,-2]$ (orange) is distant from $P$ (blue) in feature space, but $Q^*$ (purple) reweights $T$ to recover $P$'s score distribution (middle). The CTM for $Q^*$ remains non-increasing (right).}
        \label{fig:gaus-proj-9.0,-2.0}
    \end{figure}
    Figures \ref{fig:gaus-proj-5.0,5.0}-\ref{fig:gaus-proj-9.0,-2.0} illustrate the outlined procedure for $\mu_T = [5,5], [10,10],[9,-2]$. In all cases, the CTM for $T$ (whose trajectory is not shown in the plot) exceeds the floating-point range ($S_{\max}>10^{308}$), confirming immediate detection. However, the CTM for $Q^*$ follows closely the trajectory of the CTM for $P$, and in particular, never raises an alarm despite the dramatic difference in the samples shown in the left panels.

\end{document}